\documentclass[a4paper,12pt]{article}

\usepackage[latin1]{inputenc}
\usepackage{hyperref}       
\usepackage{url}   
\usepackage{amsmath}
\usepackage{color}
\usepackage{amsfonts,enumitem}
\usepackage{amssymb}
\usepackage{graphicx,tikz}
\usepackage{booktabs}       
\usepackage{amsfonts}       
\usepackage{nicefrac}       
\usepackage{microtype}

\usepackage{algorithmic}
\usepackage[ruled,noend]{algorithm2e}
\usepackage[skins]{tcolorbox}

\newcommand{\crit}{\mathrm{crit}}
\newcommand{\conv}{\mathrm{conv}}

\newcommand{\bx}{\mathbf{x}}
\newcommand{\bw}{\mathbf{w}}
\newcommand{\bz}{\mathbf{z}}
\newcommand{\bv}{\mathbf{v}}

\newcommand{\RR}{\mathbb{R}}
\newcommand{\NN}{\mathbb{N}}

\newtheorem{theorem}{Theorem}
\newtheorem{lemma}{Lemma}

\newtheorem{corollary}{Corollary}

\newtheorem{definition}{Definition}
\newtheorem{assumption}{Assumption}

\newtheorem{claim}{Claim}

\newenvironment{proof}[1][]{\noindent {\bf Proof #1:\;}}{\hfill $\Box$}

\providecommand{\keywords}[1]{\textbf{\textbf{Keywords. }} #1}

\title{Incremental Without Replacement Sampling in Nonconvex Optimization}
\begin{document}
\author{
	 Edouard Pauwels\thanks{IRIT, Universit\'e de Toulouse, CNRS. Toulouse, France.}
}

\date{\today}

\maketitle

\begin{abstract}
				Minibatch decomposition methods for empirical risk minimization are commonly analysed in a stochastic approximation setting, also known as sampling with replacement. On the other hands modern implementations of such techniques are incremental: they rely on sampling without replacement, for which available analysis are much scarcer. We provide convergence guaranties for the latter variant by analysing a versatile incremental gradient scheme. For this scheme, we consider constant, decreasing or adaptive step sizes. In the smooth setting we obtain explicit complexity estimates in terms of epoch counter. In the nonsmooth setting we prove that the sequence is attracted by solutions of optimality conditions of the problem.
\end{abstract}

Communicated by Gabriel Peyr\'e


\section{Introduction}
\subsection{Context and motivation}
Training of modern learning architectures is mostly achieved by empirical risk minimization, relying on minibatch decomposition first order methods \cite{bottou2018optimization,lecun2015deep}. The goal is to solve optimization problems of the form
\begin{align}
				F^* = \inf_{x \in \RR^p} F(x) = \frac{1}{n} \sum_{i=1}^n f_i(x)
				\label{eq:mainProblem}
\end{align}
where $f_i \colon \RR^p \colon \to \RR$ are Lipschitz functions and the infimum is finite.
In this context minibatching takes advantage of redundancy in large sums and perform steps which only rely on partial sums \cite{bottou2008tradeoffs}. The most widely studied variant is the Stochastic Gradient algorithm (also known as SGD), each step consists in sampling with replacement in $\{1,\ldots, n\}$, and moving in the direction of the gradient of $F$ corrupted by centered noise inherent to subsampling. This allows to study such algorithms in the broader context of stochastic approximation, initiated by Robins and Monro \cite{robbins1951stochastic} with many subsequent works \cite{ljung1977analysis,benaim1999dynamics,kushner20003stochastic,borkar2009stochastic,moulines2011nonasymptotic,davis2018stochastic,bolte2020conservative}.

On the other hand most widely used implementations of such learning strategies for deep network \cite{abadi2016tensorflow,paszke2017workshop} rely on sampling without replacement, an epoch being the result of a single path during which first order information for each $f_i$ is computed exactly once. Although very close to stochastic approximation, this strategy does not satisfy the ``gradient plus centered noise'' hypothesis. Therefore all existing theoretical guaranties relying on stochastic approximation arguments do not hold true for many practical implementation of learning algorithm. The purpose of this work is to provide convergence guaranties for such ``without replacement minibatch strategies'' for problem \eqref{eq:mainProblem}, also known as incremental methods \cite{bertsekas2011incremental}. The main strategy is to view the algorithmic steps throught the lenses of perturbed gradient iterations, see for example \cite{bertsekas2000gradient}.

\subsection{Problem setting}
We consider problem \eqref{eq:mainProblem} and assume that for each $f_i$ is Lipschitz and that we have access to an oracle which provides a search direction $d_i \colon \RR^p \mapsto \RR^p$, $i=1, \ldots,n$. 
We consider two settings

\textbf{Smooth setting:} $f_i$ are $C^1$ with Lipschitz gradient, in which case we set $d_i = \nabla f_i$, $i=1, \ldots, n$.

\textbf{Nonsmooth setting:} $f_i$ are path differentiable. Such functions constitute a subclass of Lipschitz functions which enjoy some of the nice properties of nonsmooth convex functions \cite{bolte2020conservative} in particular, an operational chain rule \cite{davis2018stochastic}. In this case, $d_i$ is a selection in a conservative field for $f_i$. Examples of such objects include convex subgradients if $f_i$ is convex, the Clarke subgradient, which is an extension to the nonconvex setting, as well as the output of automatic differentiation applied to a nonsmooth program, see \cite{bolte2020conservative} for more details.

We consider a class of descent methods described by Algorithm \ref{alg:prototypeAlgorithm}. Notice that there is no randomness specified in the algorithm, all our results are worst case and hold deterministically. 
Algorithm \ref{alg:prototypeAlgorithm} allows to model:
\begin{algorithm}[t]
  \caption{Without replacement descent algorithm}
	\label{alg:prototypeAlgorithm}
	\SetAlgoLined
	\textbf{Program data:} For $i = 1 \ldots, n$, $f_i$,  corresponding search direction oracle $d_i$.\\
	\textbf{Input:} $x_0 \in \RR^p$.\\
	\begin{minipage}[h]{0.44\linewidth}
	\begin{algorithmic}[1]
					\STATE \textbf{Decreasing steps:} 
					\STATE $\left( \alpha_{K,i} \right)_{K \in \NN, i \in \{1,\ldots, n\}}$.
					\FOR{$K \in \NN$}
    			\STATE Set: $z_{K,0} = x_K$
					\FOR{$i = 1,\ldots,n$}
									\STATE $\hat{z}_{K,i-1} \in \mathrm{conv} \left((z_{K,j})_{j=0}^{i-1}\right)$.
									\STATE $z_{K,i} = z_{K,i-1} - \alpha_{K,i} d_i(\hat{z}_{K,i-1})$ 
					\ENDFOR
					\STATE Set: $x_{K+1} = z_{K,n}$ 
					\ENDFOR
	\end{algorithmic}
	\end{minipage}\quad
	\begin{minipage}[h]{0.44\linewidth}
	\begin{algorithmic}[1]
					\STATE \textbf{Adaptive steps:} 
					\STATE $v_0 = \delta > 0$, $\beta > 0$. \\
					\FOR{$K \in \NN$}
    			\STATE Set: $z_{K,0} = x_K$, $v_{K,0} = v_{K}$ 
					\FOR{$i = 1,\ldots,n$}
									\STATE $\hat{z}_{K,i-1} \in \mathrm{conv} \left((z_{K,j})_{j=0}^{i-1}\right)$.
									\STATE $v_{K,i} = v_{K,i-1} + \beta \|d_i(\hat{z}_{K,i-1})\|^2$
									\STATE $\alpha_{K,i} = v_{K,i}^{-1/3}$
									\STATE $z_{K,i} = z_{K,i-1} - \alpha_{K,i} d_i(\hat{z}_{K,i-1})$ 
					\ENDFOR
					\STATE Set: $x_{K+1} = z_{K,n}$, $v_{K+1} = v_{K,n}$.
					\ENDFOR
	\end{algorithmic}
	\end{minipage}
\end{algorithm}

\textbf{Gradient descent:} Set $\hat{z}_{K,i-1} = z_{K,0} = x_K$, for all $K \in \NN$ and $i = 1 \ldots, n$.

\textbf{Incremental algorithms:} Set $\hat{z}_{K,i-1} = z_{K,i-1}$, for all $K \in \NN$ and $i = 1 \ldots, n$.

\textbf{Random permutations:} Although not explicitely stated in the algorithm, all our proof argument hold independantly of the order of query of the indices $i = 1,\ldots, n$ for each epoch. Hence our results actually hold deterministically for ``random shuffling'' or, ``without replacement sampling'' strategies.

\textbf{Mini-batching:} Set $\hat{z}_{K,i-1} = \hat{z}_{K,i-2} = z_{K,i-2}$, which results in computation of gradients of $f_i$ and $f_{i-1}$ at the same point $z_{K,i-2}$.

\textbf{Asynchronous computation in a parameter server setting:} Consider that $(z_{K,i})_{K \in \NN,i=1\ldots n}$ is stored on a server, accessed by workers which compute $d_i(z_{K,i-1})$. Due to communication and computation delays, $d_i$ may be evaluated using an outdated estimate of $z$, called $\hat{z}$. In Algorithm \ref{alg:prototypeAlgorithm}, asynchronicity and delays between workers may be arbitrary within each epoch. However, we enforce that the whole system waits for all workers to communicate results before starting a new epoch, a form of partial synchronization.

\subsection{Contributions}
We propose a detailed convergence analysis of Algorithm \ref{alg:prototypeAlgorithm} in a nonconvex setting. Our analysis is worst case and our results hold deterministically. When each $f_i$ is smooth with $L_i$-Lipschitz gradient, setting $L = \frac{1}{n} \sum_{i=1}^n L_i$, we obtain the following estimates on the squared norm of the gradient of $F$ in terms of number of epochs $K$ (all constants are explicit). 
\begin{itemize}
		\item Decreasing step size without knowledge of $L$: $O\left(\frac{1}{\sqrt{K}}\right)$.
		\item Decreasing step size with knowledge of $L$: $O\left(\frac{1}{K^{2/3}}\right)$.
		\item Adaptive step size without knowledge of $L$: $O\left(\frac{1}{K^{2/3}}\right)$.
\end{itemize}
For genereral nonsmooth objectives, convergence rate do not exist for the simplest subgradient oracle, see for example \cite{zhang2020complexity} with an attempt for more complex oracles. We prove that the sequence $(x_K)_{K \in \NN}$ is attracted by subsets of $\RR^p$ which are solutions to optimality condition related to problem \eqref{eq:mainProblem}, for both step size strategies.

\subsection{Relation to existing litterature}
Incremental gradient was introduced by Bertsekas in the late 90's \cite{bertsekas1997new}, extended with a gradient plus error analysis \cite{bertsekas2000gradient} and nonsmooth version \cite{nedic2001incremental}. An overview is given in \cite{bertsekas2011incremental}, see also \cite{bertsekas2015convex}. Most convergence analyses are qualitative and limited to convex objectives, only few rates are available. The prescribed step size strategy in Algorithm \ref{alg:prototypeAlgorithm} is direcly inspired from these works. We analyse the incremental method as a perturbed gradient method, a view which was exploited in \cite{bertsekas2000gradient,mania2017perturbed} and in distributed settings, see for example \cite{mcmahan2017communication,mishchenko2018delay,lan2018communication,pu2020distributed}.

The idea of the adaptive step size is taken from the Adagrad algorithm introduced in \cite{duchi2011adaptive}. Analysis of such algorithms for nonconvex objectives was proposed in  \cite{li2019convergence,ward2019adagrad,barakat2018convergence,defossez2020convergence} in the stochastic and smooth setting. To our knowledge the combination of adaptive step sizes with incremental methods has not been considered. We use the ``scalar step variant'' of the Adagrad, called Adagrad-norm in \cite{ward2019adagrad} or global step size in \cite{li2019convergence}, in contrast with the originally proposed coordinatewise step sizes analysed in \cite{defossez2020convergence}. The original Adagrad algorithm has a power $1/2$ in the denominator which we replaced by $1/3$ in order to obtain faster rates, taking advantage of smoothness and the finite sum structure.

It has been a longstanding open question in machine learning to investigate the advantages of random permutations compared to vanilla SGD \cite{bottou2009curiously,recht2012towar}. The main motivation is that random permutations often outperforms with replacement sampling despite the absence of theory to explain this observation. The topic is still active and recent progresses has been made in the strongly convex setting, see \cite{ying2018stochastic,gurbuzbalaban2019why,rajput2020closing,safran2020good} and reference therein. Rather than studying the superiority of random permutations in the nonconvex setting, we consider the more modest goal of proving convergence guaranties for such strategies. This is achieved by following a perturbed iterate view, a strategy which provides worst case guaranties which are of a different nature compared to average case or almost sure guaranties commonly obtained for stochastic approximation algorithms. The obtained rates have a worse dependency in $n$ compared to SGD but are asymptotically faster than the best known rate for SGD. Similar complexity estimates were obtained in \cite{nguyen2020unified,mishchenko2020random} for prescribed step size strategies. To our knowledge, the adaptive variant has not been treated.

Our nonsmooth convergence analysis relies on the ODE method, see \cite{ljung1977analysis} with many subsequent developments \cite{benaim1999dynamics,kushner20003stochastic,benaim2005stochastic,borkar2009stochastic,barakat2018convergence}. In particular we build uppon a nonsmooth ODE formulation, differential inclusions \cite{clarke1983optimization,aubin1984differential}. This was used in \cite{davis2018stochastic} to analyse the stochastic subgradient algorihtm in nonconvex settings using the subgradient projection formula \cite{bolte2007clarke}. In the nonsmooth world the backpropagation algorithm \cite{rumelhart1986learning} used in deep learning suffers from inconsistent behaviors and may not provide subgradient of any kind \cite{griewank2008evaluating,kakade2018provably}. We use the recently introduced tool of conservative fields and path differentiable function \cite{bolte2020conservative} capturing the full complexity of backpropagation oracles. Our proof essentially relies on the notion of Asymptotic Pseudo Trajectory (APT) for differential inclusions\cite{benaim1996apt,benaim2005stochastic}.

\subsection{Preliminary results}
It is important to emphasize that in Algorithm \ref{alg:prototypeAlgorithm}, the adaptive step strategy is a special case of the prescribed step strategy. Hence our analysis will start by general considerations for the prescribed step strategy followed by specific considereations to the adaptive steps. We start with a simple claim whose proof is given in appendix \ref{sec:proofsSmooth} and provides a bound on the length of the steps taken by the algorithm.
\begin{claim}
				For all $K \in \NN$ and all $i = 1, \ldots, n$, we have 
				\begin{align}
								\max\left\{ \|z_{K,i} - x_K\|^2, \|x_{K+1} - x_K\|^2, \|\hat{z}_{K,i-1} - x_K\|^2\right\} &\leq n \sum_{i=1}^n \alpha_{K,i}^2 \|d\left( \hat{z}_{K,i-1} \right)\|^2, 
								\label{eq:iterateBound}
				\end{align}
				\label{claim:iterateBound}
\end{claim}
Throughout this paper, we will work under decreasing step size condition whose meaning is described in the following assumption.
We remark that both step size strategies provided in Algorithm \ref{alg:prototypeAlgorithm} comply with this constraint. 
\begin{assumption}
				The sequence $(\alpha_{K,i})_{K \in \NN, i \in \{1,\ldots, n\}}$ is non increasing with respect to the lexicographic order. That is for all $K \in \NN$, $i = 2, \ldots, n$, $\alpha_{K,i-1} \geq \alpha_{K,i} \geq \alpha_{K+1,1}$.
				\label{ass:stepSizeSmooth}
\end{assumption}
\section{Quantitative analysis in the smooth setting}
In this section we consider that each $f_i$ has Lipschitz gradient, in which case, $d_i$ is set to be $\nabla f_i$. Note that in this setting, $\nabla F = \frac{1}{n} \sum_{i=1}^n \nabla f_i$ and $F$ also has $L$-Lipschitz gradient.
\begin{assumption}
				For $i = 1, \ldots, n$, 
				\begin{itemize}
								\item $f_i \colon \RR^p \colon \to \RR$ is an $M_i$ Lipschitz functions and $d_i \colon \RR^p \mapsto \RR^p$ is such that for all $x \in \RR^p$, $\|d_i(x)\|\leq M_i$. We let $M = \sqrt{\frac{1}{n}\sum_{i=1}^n M_i^2}$. Note that in this case $F$ is $M$-Lipschitz using Lemma \ref{lem:techLemma} in appendix \ref{sec:lemmas} which allows to bound $\|\sum_{i=1}^n d_i\|$.
								\item $f_i$ is continuously differentiable with $L_i$ Lipschitz gradient and we set $d_i = \nabla f_i$. We set $L = \frac{1}{n} \sum_{i=1}^n L_i$. Note that in this case $F$ has $L$-Lipschitz gradient as shown in Claim \ref{claim:FSmooth}.
				\end{itemize}
				\label{ass:smoothnessAssumption}
\end{assumption}

The technical bulk of our analysis is given by the following claim whose proof is provided in appendix \ref{sec:proofsSmooth}. Note that this result holds deterministically and independantly of the considered step size strategy.
\begin{claim}
				Under Assumptions \ref{ass:stepSizeSmooth} and \ref{ass:smoothnessAssumption}, for all $K \in \NN$, $K\geq1$, setting $\alpha_K = \alpha_{K-1,n}$ and $\alpha_0 = \delta^{-1/3} \geq \alpha_{0,1}$, denoting by $(\cdot)_+$ the positive part, we have 
				\begin{align}
								&F(x_{K+1}) - F(x_K) + \frac{n \alpha_K}{2} \|\nabla F(x_K) \|^2 
								\label{eq:mainInequalitySmooth}\\
								\leq\,&  \left(\alpha_K L^2n^2 + \left(\frac{Ln}{2} - \frac{1}{2 \alpha_K}\right)_+\right)  \sum_{j=1}^n \alpha_{K,j}^2\|d_j(\hat{z}_{K,j-1})\|^2 + \alpha_K M^2 \sum_{i=1}^n \left( 1 - \frac{\alpha_{K,i}^3}{\alpha_K^3} \right).\nonumber
				\end{align}
				\label{claim:mainIneqSmooth}
\end{claim}

\subsection{Prescribed step size without knowledge of $L$}
The following holds under Assumption for Algorithm \ref{alg:prototypeAlgorithm}.
\begin{corollary}
				If the step size is constant, $\alpha_{K,i} = \alpha/n$ for all $K \in \NN$, $i = 1 \ldots, n$, we have
				\begin{align*}
								\min_{K=0,\ldots, N} \|\nabla F(x_K) \|^2 \leq \frac{2(F(x_0) - F^*) }{(N+1)\alpha } +  2\left(\alpha L^2M^2 + \frac{LM^2}{2}\right)  \alpha 
				\end{align*}
				\label{cor:constantSmoothNoL}
\end{corollary}

\begin{corollary}
				If the step size is decreasing $\alpha_{K,i} = 1 / (n\sqrt{K+1})$, for all $K \in \NN$, $i = 1 \ldots, n$, then
				\begin{align*}
								\min_{K=1,\ldots, N} \|\nabla F(x_K) \|^2 \leq\frac{1}{\sqrt{N+1} - 1} \left(F(x_0) - F^*  +  \left(L^2M^2 + \frac{LM^2}{2}\right)  \left( 1 + \log(N+1) \right)  \right)
				\end{align*}
				\label{cor:decreasingStepSizeNoL}
\end{corollary}

\subsection{Prescribed step sizes based on $L$}
\label{sec:corollaries}
The following hold under Assumption \ref{ass:smoothnessAssumption} for Algorithm \ref{alg:prototypeAlgorithm}.
\begin{corollary}
				If the step size is constant, $\alpha_{K,i} = \alpha/n$, with $\alpha \leq 1/L$, then for all $K \in \NN$, $i = 1 \ldots, n$, we have
				\begin{align*}
								\min_{K=0,\ldots, N} \|\nabla F(x_K) \|^2 \leq \frac{2(F(x_0) - F^*) }{(N+1)\alpha } +  2\alpha^2 L^2M^2  
				\end{align*}
				\label{cor:constantSmooth}
\end{corollary}

\begin{corollary}
				If the step size is decreasing $\alpha_{K,i} = 1 / (Ln(K+1)^{1/3})$, for all $K \in \NN$, $i = 1 \ldots, n$, then
				\begin{align*}
								\min_{K=1,\ldots, N} \|\nabla F(x_K) \|^2 \leq\frac{2}{3((N+1)^{2/3} - 1)} \left(L(F(x_0) - F^*)  +  M^2  \left( 1 + \log(N+1) \right)  \right)
				\end{align*}
				\label{cor:decreasingStepSize}
\end{corollary}

\subsection{Adaptive step size}
The following hold under Assumption \ref{ass:smoothnessAssumption} for Algorithm \ref{alg:prototypeAlgorithm}.
\begin{corollary}
				If we consider the adaptive step size strategy with $\beta = n^2$ and $\delta = n^{3}$, then 
				\begin{align*}
								&\min_{K=0,\ldots, N} \|\nabla F(x_K) \|^2 \\
								\leq\quad &2 (M^2 +1)^{1/3} \frac{F(x_0) - F^*  + \left( L^5 + \frac{L^4}{2} \right)  + \left( \frac{L^2}{2} (1 + M)^{1/3} + M^2 \right) \log\left( 1 + M^2 (N+1) \right)}{(N+1)^{2/3}}.
				\end{align*}
				\label{cor:adaptive}
\end{corollary}

\subsection{Discussion on the obtained convergence rates}
All the complexity estimates decribed in Section \ref{sec:corollaries} are given in terms of $K$, which is the number of epochs. 
In particular, there is no dependency in the size of the sum $n$ or in the dimension $p$ beyond problem constants $L$ and $M$. 
The work presented in \cite{ghadimi2013stochastic}, see also Theorem 1 in \cite{reddi2016fast}, ensures that the convergence rate of ``with replacement'' SGD applied to problem \ref{eq:mainProblem} is of order $O(1/\sqrt{k})$ under the same assumptions as ours, where $k$ is the number of stochastic iteration (typically $n$ times bigger than the numer of epochs).
From this perspective, the dependency in $n$ is unfavorable as our rates are in terms of number of epochs rather than number of iterations which is customary in stochastic settings \cite{bottou2008tradeoffs,moulines2011nonasymptotic,bottou2018optimization}. One element of explaination is the nature of our perturbed analysis, which is worst case and blind to the order in which elements are chosen, in contrast with average case stochastic analysis usually performed when considering ``with replacement'' strategies. This is an important issue, since in practice, for example for deep learning problems, only a few epochs are performed on large datasets. Furthermore, SGD naturally accomodates stochastic data augmentation commonly used in deep learning contexts, a property which is not shared by incremental algorithms.

On the other hand, for prescribed step size and adaptive step size, the convergence rate is of the order of $K^{-2/3}$ which is asymptotically faster than SGD which would be of the order $(nK)^{-1/2}$. This result is only based on comparison of upper bounds and holds only asymptotically since the proposed rate gets better for $K \geq n^3$ which is a regime not considered in practical applications. It constitutes an advantage of the proposed incremental scheme but not a proof of its superiority compared to SGD. Similar rates were obtained in \cite{nguyen2020unified} and in \cite{mishchenko2020random}, with an improved dependency in $n$ and weaker boundedness assumptions. In both cases, it is required to know the Lipschitz constant $L$ which is hardly accessible in practice. The adaptive variant removes this requirement while maintaining a similar rate, showing the advantage of adaptive step sizes in this context. Furthermore, the proposed numerical scheme is more versatile than algorithms in \cite{nguyen2020unified,mishchenko2020random} as it allows for a unified treatment of certain form of delays such as minibatching or limited asynchronicity. Interestingly, a cube root variation of the adaptive step size was introduced in \cite{defazio2021adaptivity} in combination with coordinatewise updates and momentum. The purpose is however different, it allows to maintain an effective step size in presence of advanced weighting schemes and is analyzed in a convex setting, while the analysis proposed here takes advantage of larger steps (compared to the vanilla Adagrad algorithm) to obtain faster rates in a nonconvex setting.

Regarding our assumptions, Lipschicity and boundedness of gradients in Assumption \ref{ass:smoothnessAssumption} are common in the analysis of stochastic gradient schemes in a nonconvex context, see for examle \cite[Theorem 1]{reddi2016fast} for SGD and more recently for adaptive variants \cite{defossez2020convergence} and incremental variants \cite{nguyen2020unified,mishchenko2020random}. Note that in the stochastic approximation context, proxies are often used for these assumptions, requiring them only to hold on the whole sum in \eqref{eq:mainProblem} rather than on each element. This is often complemented by a uniformly bounded variance assumption, see for example \cite{ward2019adagrad}. In finite sum contexts, all these assumptions are very close in nature, as smoothness of the sum directly relates to smoothness of its components. It is worth mentioning that these boundedness assumptions could be relaxed to hold only locally if it is assumed that the sequence remains bounded.

\subsection{Proofs for the obtained complexity estimates}

\begin{proof}[of Corollary \ref{cor:constantSmoothNoL}]
				The considered step size complies with Assumption \ref{ass:smoothnessAssumption} so that Claim \ref{claim:mainIneqSmooth} applies.
				Fix $K \in \NN$, fix $\alpha_{K,i} = \alpha_K$ for all $i = 1,\ldots,n$, we have $1 - \frac{\alpha_{K,i}^3}{\alpha_K^3} = 0$.
				Combining with Claim \ref{claim:mainIneqSmooth}, using $\alpha_K \leq \alpha_0$, for all $K \in \NN$, we have
				\begin{align*}
							\frac{n \alpha_K}{2} \|\nabla F(x_K) \|^2&\leq F(x_K) - F(x_{K+1}) + \left(\alpha_0 L^2M^2n+ \frac{LM^2}{2}\right)  n^2\alpha_{K}^2 .
				\end{align*}
				Summing for $K = 0, \ldots, N$ and dividing by $\sum_{K=0}^N n \alpha_K$, we obtain
				\begin{align}
								\min_{K=0,\ldots, N} \|\nabla F(x_K) \|^2 \leq& \frac{2}{\sum_{K=0}^N n \alpha_K } \left( F(x_0) - F^* +  \left(\alpha_0L^2M^2n + \frac{LM^2}{2}\right)  \sum_{K=0}^N n^2 \alpha_K^2 \right)
								\label{eq:mainSmoothPrescribedNoL}
				\end{align}
				Choosing constant step $\alpha / n$ for $\alpha >0$, we obtain
				\begin{align*}
								\min_{K=0,\ldots, N} \|\nabla F(x_K) \|^2 \leq\quad \frac{2(F(x_0) - F^*) }{(N+1)\alpha } +  2\left(\alpha L^2M^2 + \frac{LM^2}{2}\right)  \alpha 
				\end{align*}
\end{proof}

\begin{proof}[of Corollary \ref{cor:decreasingStepSizeNoL}]
				The considered step size complies with Assumption \ref{ass:smoothnessAssumption} so that Claim \ref{claim:mainIneqSmooth} applies.
				In this setting \eqref{eq:mainSmoothPrescribedNoL} is still valid. Choosing $\alpha_K = \frac{1}{n \sqrt{K+1}}$, we have
				\begin{align*}
								\sum_{K=0}^N n\alpha_K &\geq \int_{t=0}^{t = N+1} \frac{1}{\sqrt{t + 1}} dt \geq 2\left( \sqrt{N+1} - 1 \right) \\
								\sum_{K=0}^N n^2\alpha_K^2 &\leq  \left(1 + \sum_{K=1}^N \frac{1}{K+1} \right) \leq\left( 1 + \int_{t=0}^{t=N}\frac{dt}{t+1} \right) = \frac{1}{n^2}\left( 1 + \log(N+1) \right)
				\end{align*}
				and we obtain in \eqref{eq:mainSmoothPrescribedNoL}
				\begin{align*}
								\min_{K=1,\ldots, N} \|\nabla F(x_K) \|^2 \leq\quad&  \frac{1}{\sqrt{N+1} - 1} \left(F(x_0) - F^*  +  \left(L^2M^2 + \frac{LM^2}{2}\right)  \left( 1 + \log(N+1) \right)  \right)
				\end{align*}
\end{proof}

\begin{proof}[of Corollary \ref{cor:constantSmooth}]
				The considered step size complies with Assumption \ref{ass:smoothnessAssumption} so that Claim \ref{claim:mainIneqSmooth} applies.
				Fix $K \in \NN$, fix $\alpha_{K,i} = \alpha_K$ for all $i = 1,\ldots,n$, we have $1 - \frac{\alpha_{K,i}^3}{\alpha_K^3} = 0$.
				Combining with Claim \ref{claim:mainIneqSmooth}, we have, using the fact that $\alpha_K \leq 1/(Ln)$, 
				\begin{align*}
							\frac{n \alpha_K}{2} \|\nabla F(x_K) \|^2&\leq F(x_K) - F(x_{K+1}) + L^2M^2  n^3\alpha_{K}^3 .
				\end{align*}
				Summing for $K = 0, \ldots, N$ and dividing by $\sum_{K=0}^N n \alpha_K$, we obtain
				\begin{align}
								\min_{K=0,\ldots, N} \|\nabla F(x_K) \|^2 \leq& \frac{2}{\sum_{K=0}^N n \alpha_K } \left( F(x_0) - F^* +  L^2M^2  \sum_{K=0}^N n^3 \alpha_K^3 \right)
								\label{eq:mainSmoothPrescribed}
				\end{align}
				Choosing constant step $\alpha / n$ for $\alpha >0$, we obtain
				\begin{align*}
								\min_{K=0,\ldots, N} \|\nabla F(x_K) \|^2 \leq\quad \frac{2(F(x_0) - F^*) }{(N+1)\alpha } +  2\alpha^2 L^2M^2 
				\end{align*}
\end{proof}

\begin{proof}[of Corollary \ref{cor:decreasingStepSize}]
				The considered step size complies with Assumption \ref{ass:smoothnessAssumption} so that Claim \ref{claim:mainIneqSmooth} applies.
				In this setting \eqref{eq:mainSmoothPrescribed} is still valid. Indeed, choosing $\alpha_K = \frac{1}{Ln (K+1)^{1/3}}$, we have $\alpha_K \leq 1/ (Ln)$ for all $K \in \NN$. Furthermore, 
				\begin{align*}
								\sum_{K=0}^N n\alpha_K &\geq \int_{t=0}^{t = N+1} \frac{1}{L(t + 1)^{1/3}} dt \geq \frac{3}{2L}\left( (N+1)^{2/3} - 1 \right) \\
								\sum_{K=0}^N n^3\alpha_K^3 &\leq \frac{1}{L^3} \left(1 + \sum_{K=1}^N \frac{1}{K+1} \right) \leq \frac{1}{L^3}\left( 1 + \int_{t=0}^{t=N}\frac{dt}{t+1} \right) = \frac{1}{L^3}\left( 1 + \log(N+1) \right)
				\end{align*}
				and we obtain in \eqref{eq:mainSmoothPrescribed}
				\begin{align*}
								\min_{K=1,\ldots, N} \|\nabla F(x_K) \|^2 \leq\quad&  \frac{2}{3((N+1)^{2/3} - 1)} \left(L(F(x_0) - F^*)  +  M^2  \left( 1 + \log(N+1) \right)  \right)
				\end{align*}
\end{proof}

\begin{proof}[of Corollary \ref{cor:adaptive}]
				The considered step size complies with Assumption \ref{ass:smoothnessAssumption} so that Claim \ref{claim:mainIneqSmooth} applies.
				We write for all $K \in \NN$ and all $i =1 ,\ldots, n$, $\alpha_{K} = v_{K}^{-1/3}$. Let us start with the following.
				\begin{claim}
								For all $K \in \NN$ and all $i =1 ,\ldots, n$
								\begin{align}
												1-\frac{\alpha_{K,i}^3}{\alpha_{K}^3} \leq \beta\sum_{j=1}^n \frac{\|d_j(\hat{z}_{K,j-1})\|^2}{v_{K,j}}
												\label{eq:smoothAdaTemp41}
								\end{align}
								\label{claim:smoothAdaTemp1}
				\end{claim}
				\begin{proof}[of claim \ref{claim:smoothAdaTemp1}]
								Fix $K \in \NN$ and $i$ in $1,\ldots, n$, we have
								\begin{align*}
												v_{K} \leq v_{K,i} = v_K + \beta \sum_{j=1}^i \|d_j(\hat{z}_{K,j-1})\|^2.
								\end{align*}
								From this we deduce, using the fact that $v_{K,j}$ is non decreasing in $j$,
								\begin{align*}
												1 - \frac{\alpha_{K,i}^3}{\alpha_K^3} &=\frac{v_{K,i} -v_K}{v_{K,i}} =\frac{\beta\sum_{j=1}^i \|d_j(\hat{z}_{K,j-1})\|^2}{v_{K,i}} \leq \beta\sum_{j=1}^i \frac{\|d_j(\hat{z}_{K,j-1})\|^2}{v_{K,j}} \leq \beta\sum_{j=1}^n \frac{\|d_j(\hat{z}_{K,j-1})\|^2}{v_{K,j}},
								\end{align*}
				\end{proof}
				
				Combining Claim \ref{claim:mainIneqSmooth} and Claim \ref{claim:smoothAdaTemp1}, we have for all $K \in \NN$, using $\delta^{-1/3} \geq \alpha_K$
				\begin{align}
								\frac{n \alpha_K}{2} \|\nabla F(x_K) \|^2 &\leq F(x_K) -  F(x_{K+1})  + \left(\alpha_K L^2n^2 + \left(\frac{Ln}{2} -\frac{1}{2\alpha_K}\right)_+\right)  \sum_{j=1}^n \alpha_{K,j}^2\|d_j(\hat{z}_{K,j-1})\|^2 \nonumber\\
								&+ M^2 n \frac{\beta}{\delta^{1/3}} \sum_{j=1}^n \alpha_{K,j}^3\|d_j(\hat{z}_{K,j-1})\|^2
								\label{eq:smoothAda1}
				\end{align}
				We will consider the following notation $\bar{\alpha} =  \frac{1}{Ln}$, we have
				\begin{align*}
								\left(\alpha_K L^2n^2 + \left(\frac{Ln}{2} -\frac{1}{2\alpha_K}\right)_+\right) \leq \alpha_K L^2n^2
				\end{align*}
				if and only if $\alpha_K \leq \bar{\alpha}$. Set $\bar{K}$, the first index $K$ such that $\alpha_{K} \leq \bar{\alpha}$. For all $K \leq \bar{K} - 1$, we have $1 / \delta^{1/3} \geq \alpha_K = v_K^{-1/3} > \bar{\alpha}$. Fix $N \leq \bar{K} - 1$, summing the second term of \eqref{eq:smoothAda1} for $K = 0 \ldots N$, we have
				\begin{align}
								&\sum_{K=0}^{N} \left(\alpha_K L^2n^2 + \left(\frac{Ln}{2} -\frac{1}{2\alpha_K}\right)_+\right)  \sum_{j=1}^n \alpha_{K,j}^2\|d_j(\hat{z}_{K,j-1})\|^2  \nonumber \\
								\leq \quad&  \left( \frac{L^2n^2}{\delta} + \frac{Ln}{2 \delta^{2/3}} \right)  \sum_{K=0}^{N} \sum_{j=1}^n \|d_j(\hat{z}_{K,j-1})\|^2 \nonumber\\
								\leq \quad& \left( \frac{L^2n^2}{\beta\delta} + \frac{Ln}{2\beta \delta^{2/3}} \right)  v_{\bar{K}} \nonumber \nonumber\\
								\leq \quad& \left( \frac{L^2n^2}{\beta\delta} + \frac{Ln}{2\beta \delta^{2/3}} \right)  \frac{1}{\bar{\alpha}^3} 
								\label{eq:smoothAda12}
				\end{align}
				Now, choosing $N \geq \bar{K}$, summing the same quantity for $K \geq \bar{K}$, we have using the definition of $\bar{\alpha}$
				\begin{align}
								&\sum_{K=\bar{K}}^{N} \left(\alpha_K L^2n^2 + \left(\frac{Ln}{2} -\frac{1}{2\alpha_K}\right)_+\right)  \sum_{j=1}^n \alpha_{K,j}^2\|d_j(\hat{z}_{K,j-1})\|^2  \nonumber \\
								\leq \quad & \sum_{K=\bar{K}}^{N} \alpha_K L^2n^2 \sum_{j=1}^n \alpha_{K,j}^2\|d_j(\hat{z}_{K,j-1})\|^2  \nonumber\\
								\leq \quad & L^2n^2\sum_{j=1}^n \sum_{K=0}^{N} \frac{\alpha_K}{\alpha_{K,j}}\alpha_{K,j}^3 \|d_j(\hat{z}_{K,j-1})\|^2 \nonumber\\
								\leq \quad & L^2n^2 (1 + \beta n M / \delta)^{1/3} \sum_{j=1}^n \sum_{K=0}^{N} \alpha_{K,j}^3 \|d_j(\hat{z}_{K,j-1})\|^2 \nonumber\\
								\label{eq:smoothAda122}
				\end{align}
				where the last identity follows because for all $K \in \NN$ and $j = 1 \ldots n$,
				\begin{align*}
								 \frac{\alpha_K^3}{\alpha_{K,j}^3} = \frac{v_{K,j}}{v_K} = \frac{v_K + \beta \sum_{i=1}^j \|d_j(\hat{z}_{K,j-1})\|^2 }{v_K} \leq 1 + \frac{\beta n M}{\delta}
				\end{align*}
				Combining \eqref{eq:smoothAda12} and \eqref{eq:smoothAda122}, for any $N \in \NN$, independently of its position relative to $\bar{K}$ (and even if $\bar{K} = +\infty$), we have
				\begin{align}
								&\sum_{K=0}^{N} \left(\alpha_K L^2n^2 + \left(\frac{Ln}{2} -\frac{1}{2\alpha_K}\right)_+\right)  \sum_{j=1}^n \alpha_{K,j}^2\|d_j(\hat{z}_{K,j-1})\|^2  \nonumber \\
								\leq \quad&  \left( \frac{L^2n^2}{\beta\delta} + \frac{Ln}{2\beta \delta^{2/3}} \right)  \frac{1}{\bar{\alpha}^3} +  L^2n^2(1 + \beta n M / \delta)^{1/3} \sum_{j=1}^n \sum_{K=0}^{N} \alpha_{K,j}^3 \|d_j(\hat{z}_{K,j-1})\|^2 
								\label{eq:smoothAda123}
				\end{align}			
				Given $N \in \NN$, we may sum \eqref{eq:smoothAda1} for $K = 0 \ldots, N$ combined with \eqref{eq:smoothAda123} to obtain 
				\begin{align}
								\sum_{K=0}^{N}  \frac{n \alpha_K}{2} \|\nabla F(x_K) \|^2 &\leq F(x_0) -  F(x_N)  + \left( \frac{L^2n^2}{\beta\delta} + \frac{Ln}{2\beta \delta^{2/3}} \right)  \frac{1}{\bar{\alpha}^3} \nonumber\\
								&+ \left( L^2n^2 (1 + \beta n M / \delta)^{1/3} + M^2 n \frac{\beta}{\delta^{1/3}} \right) \sum_{j=1}^n \sum_{K=0}^{N} \alpha_{K,j}^3 \|d_j(\hat{z}_{K,j-1})\|^2 
								\label{eq:smoothAda1222}
				\end{align}
				Now, we use the lexicographic order on pairs of integers, $(a,b) \leq (c,d)$ if $a < c$ or $a = c$ and $b \leq d$. From Lemma \ref{lem:sumStepSizeLog}  in appendix \ref{sec:lemmas}, we have
				\begin{align}
								&\sum_{K=0}^N\sum_{i=1}^n \alpha_{K,i}^3\|d_i(\hat{z}_{K,i-1})\|^2 = \sum_{(K,i) \leq (N,n)} \frac{\|d_i(\hat{z}_{K,i-1})\|^2}{\delta + \beta \sum_{(k,j) \leq (K,i)} \|d_j(\hat{z}_{k,j-1})\|^2} \nonumber \\
								\leq\quad& \frac{1}{\beta} \log\left( 1 + \frac{\beta \sum_{(K,i) \leq (N,n)} \|d_i(\hat{z}_{K,i-1})\|^2}{\delta} \right) \leq \frac{1}{\beta} \log\left( 1 + \frac{\beta n M^2 (N+1)}{\delta} \right),
								\label{eq:summabilityAda}
				\end{align}		
				where the first inequality follows by applying Lemma \ref{lem:sumStepSizeLog}, noticing that we sum over $(N+1)n$ instances and that $\sum_{i=1}^n \|d_i\|^2 \leq nM^2$. 	
				We remark that for all $K \in \NN$, $\alpha_K \geq (K n\beta M^2 + \delta)^{-1/3}$. 
				Combining \eqref{eq:smoothAda1222} and \eqref{eq:summabilityAda}, we obtain
				\begin{align}
								&\frac{n(N+1)(N n\beta M^2 + \delta)^{-1/3}}{2} \min_{K=0,\ldots, N} \|\nabla F(x_K) \|^2 \nonumber\\
								\leq\quad & F(x_0) - F^*  + \left( \frac{L^2n^2}{\beta\delta} + \frac{Ln}{2\beta \delta^{2/3}} \right)  \frac{1}{\bar{\alpha}^3} + \left( \frac{L^2n^2}{\beta } (1 + \beta n M / \delta)^{1/3} + \frac{M^2 n}{\delta^{1/3}} \right) \log\left( 1 + \frac{\beta n M^2 (N+1)}{\delta} \right).
								\label{eq:sumADA0}
				\end{align}
				Combining \eqref{eq:sumADA0} with $\bar{\alpha} = 1 / (Ln)$, and choosing $\beta = n^2$ and $\delta = n^{3}$, we obtain
				\begin{align*}
								&\frac{(N+1)}{2(N M^2 + 1)^{1/3}} \min_{K=0,\ldots, N} \|\nabla F(x_K) \|^2 \nonumber\\
								\leq\quad & F(x_0) - F^*  + \left( L^5 + \frac{L^4}{2} \right)  + \left( \frac{L^2}{2} (1 + M)^{1/3} + M^2 \right) \log\left( 1 + M^2 (N+1) \right).
				\end{align*}
				The rest follows by noticing that $\frac{(N+1)}{2(N M^2 + 1)^{1/3}}  \geq \frac{(N+1)}{2((N+1) (M^2 + 1))^{1/3}} = \frac{(N+1)^{2/3}}{2(M^2 + 1)^{1/3}}$
\end{proof}
\section{Qualitative analysis for nonsmooth objectives}
In this section we consider nonsmooth objectives such as typical losses arising when training deep networks. Our analysis will be performed under the following standing assumption.
\begin{assumption}
				In addition to Assumption \ref{ass:stepSizeSmooth}, assume that 
				\begin{align}
								&\sum_{K = 0}^\infty \alpha_{K,1} = + \infty, \qquad\text{and}\qquad \alpha_{K,1} \underset{K \to \infty}{\to} 0,\qquad\text{and}\qquad \frac{\alpha_{K,1}}{\alpha_{K,n}} \underset{K \to \infty}{\to} 1.
								\label{eq:assumptionStepSize}
				\end{align}
				\label{ass:stepSize}
\end{assumption}
We follow the ODE approach, our arguments closely follow those developped in \cite{benaim2005stochastic}. We start by defining a continuous time piecewise affine interpolant of the sequence.
\begin{definition}
				For all $K \in \NN$, we let $\tau_K =  \sum_{k=0}^K \sum_{i=1}^n\alpha_{k,i}$. We fix the sequence given by Algorithm \eqref{alg:prototypeAlgorithm} and consider the associated Lipschitz interpolant such that $\bw \colon \RR^+ \mapsto \RR^p$ , such that $\bw(\tau_K) = x_K$ for all $K \in \NN$ and the interpolation is affine on $(\tau_K, \tau_{K+1})$ for all $K \in \NN$.
				\label{def:interpolant}
\end{definition}

\subsection{Differential inclusion setting}

The main argument in this Section is connecting the continuous time interpolant in Definition \ref{def:interpolant} and continuous dynamics. The continuous time counterpart of Algorithm \ref{alg:prototypeAlgorithm}, is $\dot{\bx} = \frac{-1}{n} \sum_{i=1}^n d_i(\bx)$, for which the right hand side is not continuous, classical Cauchy-Lipschitz type theorems for existence of solutions cannot be applied. We need to resort to a continuous extension of the right hand side, which becomes set valued, providing a weaker notion of solution. We use the recently introduced notion of conservativity \cite{bolte2020conservative} which captures the complexity of automatic differentiation oracles in nonsmooth settings \cite{bolte2020mathematical}. Recall that the set valued map $D$ is conservative for the locally Lipschitz function $f$, if it has a closed graph and for any locally Lipschitz curve $\bx \colon [0,1] \mapsto \RR^p$ and almost all $t \in [0,1]$ 
\begin{align}
				\frac{d}{dt} f(\bx(t)) = \left\langle v, \dot{\bx}(t) \right\rangle, \qquad \forall v \in D(\bx(t)).
				\label{eq:chainRule}
\end{align}
This is the counterpart to $\frac{d}{dt} f(\bx(t)) = \left\langle \nabla f(x(t)), \dot{\bx}(t) \right\rangle$ for any $C^1$ function $f$ and any $C^1$ curve $\bx$. This property is known as the chain rule of subdifferential inclusions, see for example \cite{davis2018stochastic}. The main specificity is that the property holds for almost all $t$ due to the fact that we have nondifferentiable objects, and for all possible choices in $D$ which is set valued, again due to nondifferentiability.
As shown in \cite{bolte2020conservative}, this ensures that for any such curve, one has
\begin{align*}
				f(\bx(1)) - f(\bx(0))  = \int \max_{v \in  D(\bx(t))} \left\langle v, \dot{\bx}(t) \right\rangle dt = \int \min_{v \in  D(\bx(t))} \left\langle v, \dot{\bx}(t) \right\rangle dt,
\end{align*}
where the integral is understood in the Lebesgue sense.

\begin{assumption}
				For $i =1,\ldots, n$, we let $D_i$ be a conservative field for $f_i$ with $\max_{v \in D_i(x)} \|v\| \leq M$ for all $x \in \RR^p$ and $d_i \colon \RR^p \mapsto \RR^p$ is measurable such that for all $x \in \RR^p$, $d_i(x) \in D_i(x)$. We set $D = \mathrm{conv}\left(\frac{1}{n} \sum_{i=1}^n D_i  \right)$. Since conservativity is preserved under addition \cite[Corollary 4]{bolte2020conservative} $D$ is conservative for $F$, furthermore it has convex compact values and a closed graph. We set $\mathrm{crit}_F$ to be the set of $x \in \RR^p$ such that $0 \in D(x)$. 
				\label{ass:conservative}
\end{assumption}
\textbf{Main examples in deep learning:} If each $f_i$, $i = 1,\ldots, n$ is the loss associated to a sample point and a neural network architecture, assuming that $f_i$ is defined using a compositional formula involving piecewise polynomials, logarithms and exponentials (which covers most of deep network architectures), then the Clarke subgradient \cite{clarke1983optimization} is a conservative field for $f_i$. Recall that the Clarke subgradient extends the notion of convex subgradient to nonconvex locally Lipschitz functions. This was proved in \cite{davis2018stochastic} using the projection formula in \cite{bolte2020conservative,bolte2020mathematical}, see also \cite{castera2019inertial,bolte2020conservative}. In deep learning context, backpropagation may fail to provide Clarke subgradients in nonsmooth contexts \cite{griewank2008evaluating,kakade2018provably}. Nontheless, it was shown in \cite{bolte2020conservative} that backpropagation computes a conservative field. Hence our analysis applies to training of deep networks using a backpropagation oracle such as the ones implemented in \cite{abadi2016tensorflow,paszke2017workshop}.
\begin{definition}
				A solution to the differential inclusion
				\begin{align*}
								\dot{\bx} \in - D(\bx)
				\end{align*}
				with inital point $x \in \RR^p$ is a locally Lipschitz mapping $\bx \colon \RR \mapsto \RR^p$ such that $\bx(0) = x$ and for almost all $t \in \RR$, $\dot{\bx}(t) \in - D(\bx(t))$. We denote by $S_x$ the set of such solutions and by $S$ the set of all soultions with any initialization.
				\label{def:solution}
\end{definition}
Standard results in this field \cite[Chapter 2, Theorem 3]{aubin1984differential} ensure that, since $D$ has closed graph and compact convex values, for any $x \in \RR^p$ the set $S_x$ is nonempty, note that it could be non unique. 

\subsection{Main result}
The following notion was introduced in \cite{benaim2005stochastic}, see also \cite{benaim1996apt}. It captures the fact that a continuous trajectory is a solution to the differential inclusion in Definition \ref{def:solution} asymptotically. Note that we let the initialization free in the next definitionn, this is necessary to apply \cite[Theorem 4.1]{benaim2005stochastic}. 
\begin{definition}[Asymptotic pseudo trajectory]
				A continuous function $\bz \colon \RR_+ \mapsto \RR^p$ is an asymptotic pseudotrajectory (APT), if for all $T>0$, 
				\begin{align*}
								\lim_{t \to \infty} \inf_{\bx  \in S} \sup_{0 \leq s \leq T} \|\bz(t+s) - \bx(s)\| = 0.
				\end{align*}
				\label{def:apt}
\end{definition}

\begin{claim}
				Under Assumptions \ref{ass:stepSizeSmooth}, \ref{ass:stepSize} and \ref{ass:conservative}, assume that $(x_K)_{K \in \NN}$ produced by Algorithm \ref{alg:prototypeAlgorithm} with prescribed step size is bounded. Then the interpolant $\bw$ given in Definition \ref{def:interpolant} is an asymptotic pseudo trajectory as described in Definition \ref{def:apt}.
				\label{th:interpolantAPT}
\end{claim}
The proof relies on Lemma \ref{lem:trajectoryIsPerturbedDI} which shows that the iterates produced by the algorithm satisfy a perturbed differential inclusion. The technical bulk of the proof is in Theorem \ref{th:PDITAPT} which shows that perturbed differential inclusions are aymptotic pseudo trajectories. These results are described in Section \ref{sec:proofMainResultNonsmooth}, the presentation and main arguments follow the ideas presented in \cite{benaim2005stochastic}.
In order to deduce convergence of Algorithm \ref{alg:prototypeAlgorithm} from the Asymptotic pseudo trajectory property, we need the following Morse-Sard assumption. We stress that for deep network involving piecewise polynomials, logarithms and exponentials, this assumption is satisfied for both the Clarke subgradient and the backpropagation oracle \cite{bolte2007clarke,davis2018stochastic,bolte2020conservative}.
\begin{assumption}
				The function $F$ and $D$ are such that $F(\mathrm{crit}_F)$, does not contain any open interval, where $\mathrm{crit}_F$ is given in Assumption \ref{ass:conservative} and contains all $x \in \RR^p$, with $0 \in D(x)$. 
				\label{ass:MorseSard}
\end{assumption}

\begin{corollary}
				Under Assumptions  \ref{ass:stepSizeSmooth}, \ref{ass:stepSize} and \ref{ass:conservative}, assume that $(x_K)_{K \in \NN}$ produced by Algorithm \ref{alg:prototypeAlgorithm} with prescribed step size is bounded and that Assumption \ref{ass:MorseSard} holds. Then $F(x_K)$ converges to a critical value of $F$ as $K \to \infty$ and all accumulation points of the sequence are critcal points for $D$.
				\label{cor:convergence1}
\end{corollary}
\begin{proof}
				Let $\bx \colon \RR^p \mapsto \RR$ be a solution to the differential inclusion described in Definition \ref{def:solution}. Then using conservativity in \eqref{eq:chainRule}, for almost all $t \in \RR_+$, we have
				\begin{align*}
								\frac{d}{dt} F(\bx(t)) = - \min_{v \in D(\bx(t))} \|v\|^2
				\end{align*}
				Hence $F$ is a Lyapunov function for the system: it decreases along trajectory, strictly outside $\mathrm{crit}_F$. Using Claim \ref{th:interpolantAPT}, $\bw$ is an APT. Combining Assumption \ref{ass:MorseSard} with Proposition 3.27 and Theorem 4.3 in \cite{benaim2005stochastic}, all limit points of $\bw$ are contained in $\crit_F$ and $F$ is constant on this set, that is $F(\bw(t))$ converges as $t \to \infty$.
\end{proof}

\begin{corollary}
				Under Assumption \ref{ass:conservative}, assume that $(x_K)_{K \in \NN}$ produced by Algorithm \ref{alg:prototypeAlgorithm} with adaptive step size is bounded and that Assumption \ref{ass:MorseSard} holds. Then $F(x_K)$ converges to a critical value of $F$ as $K \to \infty$ and all accumulation points of the sequence are critcal points for $D$.
				\label{cor:convergence1}
\end{corollary}
\begin{proof}
				If $v_K$ converges, this means that all $d_i$ go to $0$, and all partial increments also vanish asymptotically due to Claim \ref{claim:iterateBound}. Call the set of accumulation points $\Omega \subset \RR^p$. $\Omega$ forms a compact connected subset of $\mathrm{crit}_F$, see \cite[Lemma 3.5, (iii)]{bolte2014proximal} for details. By continuity of $F$, the $F(\Omega)$ is a connected subset of $\RR$, that is an interval. By Morse-Sard assumption \ref{ass:MorseSard} it is a singleton which proves the claim. Assume otherwise that $v_K$ diverges to $+\infty $ as $K \to \infty$, in this case, the step size goes to $0$. We have
				\begin{align*}
								v_K  \leq v_{K+1} \leq v_{K} + nM
				\end{align*}
				which shows that $v_{K+1}/v_K \to 1$ as $K \to \infty$, and $\sum_{K \in \NN}\alpha_{K,1} = + \infty$ so that Assumptions \ref{ass:stepSizeSmooth} and \ref{ass:stepSize} are valid and Corollary \ref{cor:convergence1} applies.
\end{proof}

\subsection{Proof of the main result}
\label{sec:proofMainResultNonsmooth}
We extend and adapt the arguments of \cite{benaim2005stochastic}.
\begin{definition}[Local extension]
				For any $\gamma >0$, and any $x \in \RR^p$, we let $D^\gamma$ be the following local extension of $D$
				\begin{align*}
								D^\gamma(x) = \left\{ y \in \RR^p, \, y \in \frac{1}{n} \sum_{i=1}^n \lambda_i D_i(x_i) ,\, \|x - x_i\| \leq \gamma,\, |\lambda_i - 1| \leq \gamma, \, i=1,\ldots, n  \right\}.
				\end{align*}
				Note that $\lim_{\gamma \to 0} D^\gamma(x) = \frac{1}{n} \sum_{i=1}^n D_i(x)$ by graph closedness of each $D_i$ in Assumption \ref{ass:conservative}.
				\label{def:localExtension}
\end{definition}

\begin{definition}[Perturbed differential inclusion]
				A locally Lipschitz path $\bx \colon \RR_+ \mapsto \RR^p$ satisfies the perturbed differential inclusion if there exists a function $\gamma \colon \RR_+ \mapsto \RR_+$ with $\lim_{t\to \infty} \gamma(t) = 0$, such that for almost all $t \geq 0$
				\begin{align*}
								\dot{\bx}(t) \in -D^{\gamma(t)}(\bx(t))
				\end{align*}
				\label{def:perturbedSolution}
\end{definition}

\begin{lemma}
				The interpolated trajectory $\bw$ given in Definition \ref{def:interpolant} satisfies the perturbed differential inclusion in Definition \ref{def:perturbedSolution}.
				\label{lem:trajectoryIsPerturbedDI}
\end{lemma}
\begin{proof}
				The interpolated trajectory is piecewise affine so it is locally Lipschitz and differentiable almost everywhere. For each $K \in \NN$ and $i = 1,\ldots, n$, we have using Claim \ref{claim:iterateBound}
				\begin{align}
				\|x_K - \hat{z}_{K,i-1}\| \leq n \alpha_{K,1} M.
								\label{eq:trajectoryIsPerturbedDI1}
				\end{align}
				Furthermore, for all $t \in (\tau_K, \tau_{K+1})$, 
				\begin{align}
								\dot{\bw}(t) = -\sum_{i=1}^n \alpha_{K,i} d_i(\hat{z}_{K,i-1}) / (\tau_{K+1} - \tau_K) = -\frac{1}{n} \sum_{i=1}^n \lambda_i d_i(\hat{z}_{K,i-1}),
								\label{eq:trajectoryIsPerturbedDI2}
				\end{align}
				where for all $i=1, \ldots, n$, using $\alpha_{K,i} \leq \alpha_{K,1}$ and $\tau_{K+1} - \tau_K = \sum_{i=1}^n \alpha_{K,i}\geq n \alpha_{K,n}$,
				\begin{align}
								\lambda_i = \frac{n\alpha_{K,i}}{\tau_{K+1} - \tau_K} \leq n \frac{\alpha_{K,1}}{n \alpha_{K,n}} = \frac{\alpha_{K,1}}{\alpha_{K,n}}.
								\label{eq:trajectoryIsPerturbedDI3}
				\end{align}
				Hence combining \eqref{eq:trajectoryIsPerturbedDI1} and \eqref{eq:trajectoryIsPerturbedDI2}, we may consider $\gamma(t) =  \max\left\{n \alpha_{K,1} M, \left|1 - \frac{\alpha_{K,1}}{ \alpha_{K,n}}\right| \right\}$ for all $t \in (\tau_K, \tau_{K+1})$ which satisfies the desired hypothesis.
\end{proof}

The following result is the main technical part of this section. The proof follows that of \cite[Theorem 4.2]{benaim2005stochastic} and is provided in Appendix \ref{sec:proofsNonSmooth}.
\begin{theorem}
				Let $\bz$ be a perturbed differential inclusion trajectory as given in Definition \ref{def:perturbedSolution}. Then $\bz$ is an asymptotic pseudotrajectory as described in Definition \ref{def:apt}.
				\label{th:PDITAPT}
\end{theorem}

\subsection{Discussion of the obtained result}
Definition \ref{def:perturbedSolution} extends the notion of approximate differential inclusion introduced in \cite{benaim2005stochastic} to the finite sum setting. Indeed, the definition proposed in \cite{benaim2005stochastic} coincides with ours when $n = 1$. We add the flexibility to choose different approximation points for each elements of the sum which, in turn, allows to conclude regarding the output of the algorithm. A more general result was described in \cite{bolte2014proximal} in a more abstract form. It is interesting to notice that the differential inclusion approach was developped to analyze stochastic approximation algorithms because of the difficulty caused by the addition of random noise. The proposed analysis suggests that this approach is also useful to analyse deterministic algorithms as ours. 

The obtained convergence result is qualitative and completely mirrors what is obtained for SGD under similar assumptions at this level of generality \cite{davis2018stochastic,bolte2020conservative}. In terms of assumptions, analysis of stochastic approximation requires that the step size decay is proportioned to concentration of the noise. For example, under uniformly bounded variance, step sizes should be square summable. Such assumptions ensure that perturbations of dynamics induced by the noise are summable, and therefore negligible in the limit, see fore example \cite{benaim1999dynamics} for a discussion. Such an assumption is not required by our deterministic approach, the only required assumption is that the step size goes to zero in the limit and that the steps remain of the same order within an epoch. 

All the obtained results hold under the assumption that the trajectory remains bounded. This is a strong assumption which is difficult to check a priori given a problem of the form \eqref{eq:mainProblem}. This assumption is common in the analysis of stochastic approximation algorithms \cite{benaim2005stochastic,davis2018stochastic} and we are not aware of easy sufficient condition which ensures that this is the case. A simple work around would be to add a projection step on a compact convex set at the end of each epoch. This would correpond to a constrained optimization problem in place of \eqref{eq:mainProblem}. The considered notion of approximate differential inclusion in Definition \ref{def:perturbedSolution} is general enough to include this additional algorithmic step in the analysis, maintaining the qualitative convergence result without requiring the boundedness assumption, which would be automatically fulfilled.

\section{Conclusion}
We have introduced a flexible algorithmic framework for finite sums and proposed convergence guaranties in smooth and nonsmooth settings under assumptions which are qualitatively similar as in the litterature on stochastic gradient descent for such problems. The obtained result rely on a perturbed iterate analysis and are valid in a worse case sense, they have therefore a quite different nature compared to guaranties obtained for stochastic approximation algorithms. In the smooth setting we obtain quantitative rates which have worse dependency in $n$ but are asymptotically faster. The resulting complexity estimate improves over SGD in the asymptotic regime, but remains weaker for first epochs, a situation which is not uncommon in the analysis of incremental methods \cite{gurbuzbalaban2017convergence}.

A natural extension of this work would consist in providing proof arguments explaining why random permutations, as implemented in practice, often provide superior results compared to ``with replacement sampling'' in a nonasymptotic sense. This topic has been extensively studied in the strongly convex setting \cite{bottou2009curiously,recht2012towar,ying2018stochastic,gurbuzbalaban2019why,rajput2020closing,safran2020good} and it is of interest to extend these ideas to the nonconvex and possibly nonsmooth setting \cite{nguyen2020unified,mishchenko2020random}. This will be the subject of future research. Finally, another topic of interest would be to devise variants of the proposed algorithmic scheme with faster convergence rates.

{\bf Acknowledgements.} The authors acknowledge the support of ANR-3IA Artificial and Natural Intelligence Toulouse Institute, Air Force Office of Scientific Research, Air Force Material Command, USAF, under grant numbers FA9550-19-1-7026, FA9550-18-1-0226, and ANR MasDol. The author would like to thank anonymous referees for their comments which helped improve the relevance of the paper.

\newpage
\appendix

This is the appendix for ``Incremental Without Replacement Sampling in Nonconvex Optimization''. We begin with the proof of the first claim of the paper.

\begin{proof}[of Claim \ref{claim:iterateBound}]
				We have for all $K \in \NN$ and $i = 1 \ldots n$, using the recursion in Algorithm \ref{alg:prototypeAlgorithm},
				\begin{align*}
								z_{K,i} - x_K =  \sum_{j=1}^i \alpha_{K,j} d\left( \hat{z}_{K,j-1} \right).
				\end{align*}
				Using Lemma \ref{lem:techLemma}, we obtain
				\begin{align*}
								\|z_{K,i} - x_K\|^2 \leq i \sum_{j=1}^i \alpha_{K,i}^2 \|d\left( \hat{z}_{K,i-1} \right)\|^2 \leq n \sum_{i=1}^n \alpha_{K,i}^2 \|d\left( \hat{z}_{K,i-1} \right)\|^2.
				\end{align*}
				Taking $i = n$, we obtain the second inequality.
				The result follows for $\hat{z}_{K, i-1}$ because it is in $\mathrm{conv}(z_{K,j})_{j=0}^{i-1}$ and 
				\begin{align*}
								\|\hat{z}_{K,i-1} - x_K\|^2 \leq \max_{z \in \mathrm{conv}(z_{K,j})_{j=0}^{i-1}} \|z - x_K\|^2 = \max_{j = 0, \ldots ,i} \|z_{K,j} - x_K\|^2 \leq n \sum_{i=1}^n \alpha_{K,i}^2 \|d\left( \hat{z}_{K,i-1} \right)\|^2,
				\end{align*}
				where the equality in the middle follows because the maximum of a convex function over a polyhedra is achieved at vertices.
\end{proof}
\section{Proofs for the smooth setting}
\label{sec:proofsSmooth}

For all $K \in \NN$, we let $\alpha_K = \alpha_{K-1,n}$, with $\alpha_0 = \delta^{-1/3} \geq \alpha_{0,1}$.

\subsection{Analysis for both step size strategies.}
\label{sec:smoothBoth}
\begin{claim}
				We have for all $K \in \NN$,
				\begin{align}
								&\left\langle \nabla F(x_K), x_{K+1} - x_{K} \right\rangle + \frac{1}{2n \alpha_K} \|x_{K+1} - x_K\|^2\nonumber\\
								\leq\quad & - \frac{n \alpha_K}{2} \|\nabla F(x_K) \|^2 + \alpha_K L^2n^2 \sum_{j=1}^n \alpha_{K,j}^2\|d_j(\hat{z}_{K,j-1})\|^2  + \alpha_K M^2 \sum_{i=1}^n \left( \frac{\alpha_{K,i}}{\alpha_{K}} - 1 \right)^2 
								\label{eq:smoothAdaTemp8}
				\end{align}
				\label{claim:smoothAdaTemp2}
\end{claim}
\begin{proof}[of Claim \ref{claim:smoothAdaTemp2}]
				Fix $K \in \NN$, we have
				\begin{align}
								x_{K+1} - x_K = - \sum_{i=1}^n \alpha_{K,i} d_i(\hat{z}_{K,i-1}) = - \alpha_K \sum_{i=1}^n \frac{\alpha_{K,i}}{\alpha_{K}} d_i(\hat{z}_{K,i-1})
								\label{eq:smoothAdaTemp2}
				\end{align}
				Recall that $\nabla F(x_K) = \frac{1}{n} \sum_{i=1}^n d_i(x_K)$, combining with \eqref{eq:smoothAdaTemp2}, we deduce the following
				\begin{align}
								&\left\langle \nabla F(x_K), x_{K+1} - x_{K} \right\rangle + \frac{1}{2n \alpha_K} \|x_{K+1} - x_K\|^2 \nonumber\\
								=\quad& \frac{-\alpha_K}{n} \left\langle  \sum_{i=1}^n d_i(x_K), \sum_{i=1}^n \frac{\alpha_{K,i}}{\alpha_{K}} d_i(\hat{z}_{K,i-1}) \right\rangle + \frac{1}{2n \alpha_K} \|x_{K+1} - x_K\|^2\nonumber\\
								= \quad&\frac{\alpha_K}{2n} \left( \left\|\sum_{i=1}^n d_i(x_K) - \sum_{i=1}^n \frac{\alpha_{K,i}}{\alpha_{K}} d_i(\hat{z}_{K,i-1})\right\|^2 - \left\| \sum_{i=1}^n d_i(x_K)\right\|^2 - \left\| \sum_{i=1}^n \frac{\alpha_{K,i}}{\alpha_{K}} d_i(\hat{z}_{K,i-1})\right\|^2\right) \nonumber \\
								&+ \frac{1}{2n \alpha_K} \|x_{K+1} - x_K\|^2\nonumber\\
								=\quad& - \frac{n \alpha_K}{2} \|\nabla F(x_K) \|^2 + \frac{\alpha_K}{2n} \left\|\sum_{i=1}^n d_i(x_K) - \sum_{i=1}^n \frac{\alpha_{K,i}}{\alpha_{K}} d_i(\hat{z}_{K,i-1})\right\|^2 \nonumber\\
								\leq \quad& - \frac{n \alpha_K}{2} \|\nabla F(x_K) \|^2 + \frac{\alpha_K}{n} \left( \left\|\sum_{i=1}^n d_i(x_K) - \sum_{i=1}^n  d_i(\hat{z}_{K,i-1})\right\|^2 +  \left\|\sum_{i=1}^n d_i(\hat{z}_{K,i-1}) - \sum_{i=1}^n \frac{\alpha_{K,i}}{\alpha_{K}} d_i(\hat{z}_{K,i-1})\right\|^2\right),
								\label{eq:smoothAdaTemp5}
				\end{align}
				where the first two equalities are properties of the scalar product, the third equality uses \eqref{eq:smoothAdaTemp2} to drop canceling terms and the last inequality uses $\|a + b \|^2 \leq 2(\|a\|^2 + \|b\|^2)$.
				We bound each term separately, first,
				\begin{align}
								\left\|\sum_{i=1}^n d_i(x_K) - \sum_{i=1}^n  d_i(\hat{z}_{K,i-1})\right\|^2 &\leq  \left(\sum_{i=1}^n \|d_i(x_K) -d_i(\hat{z}_{K,i-1})\|\right)^2 \nonumber \\
								& \leq \left(\sum_{i=1}^n L_i\|x_K -\hat{z}_{K,i-1}\|\right)^2 \nonumber\\
								& \leq \max_{i=1,\ldots, n}\|x_K -\hat{z}_{K,i-1}\|^2 \left( \sum_{i=1}^n L_i \right)^2 \nonumber\\
								&\leq L^2n^3 \sum_{j=1}^n \alpha_{K,j}^2 \|d_j(\hat{z}_{K,j-1})\|^2.
								\label{eq:smoothAdaTemp6}
				\end{align}
				where the first step uses the triangle inequality, the second step uses $L_i$ Lipschicity of $d_i$, the third step is H\"older inequality, and the fourth step uses Claim \ref{claim:iterateBound}.
				Furthermore, we have using the triangle inequality and Cauchy-Schwartz inequality,
				\begin{align}
								\left\|\sum_{i=1}^n d_i(\hat{z}_{K,i-1}) - \sum_{i=1}^n \frac{\alpha_{K,i}}{\alpha_{K}} d_i(\hat{z}_{K,i-1})\right\|^2 & \leq \left(\sum_{i=1}^n \left( \frac{\alpha_{K,i}}{\alpha_{K}} - 1 \right) \|d_i(\hat{z}_{K,i-1})\|\right)^2 \nonumber \\
								&\leq \sum_{i=1}^n \left( \frac{\alpha_{K,i}}{\alpha_{K}} - 1 \right)^2 \sum_{i=1}^n  \|d_i(\hat{z}_{K,i-1})\|^2\nonumber\\
								&\leq \sum_{i=1}^n \left( \frac{\alpha_{K,i}}{\alpha_{K}} - 1 \right)^2 \sum_{i=1}^n  M_i^2 \nonumber\\
								&= nM^2 \sum_{i=1}^n \left( \frac{\alpha_{K,i}}{\alpha_{K}} - 1 \right)^2
								\label{eq:smoothAdaTemp7}
				\end{align}
				Combining \eqref{eq:smoothAdaTemp5}, \eqref{eq:smoothAdaTemp6} and \eqref{eq:smoothAdaTemp7}, we obtain,
				\begin{align*}
								&\left\langle \nabla F(x_K), x_{K+1} - x_{K} \right\rangle + \frac{1}{2n \alpha_K} \|x_{K+1} - x_K\|^2\nonumber\\
								\leq\quad & - \frac{n \alpha_K}{2} \|\nabla F(x_K) \|^2 + \alpha_K L^2n^2  \sum_{j=1}^n \alpha_{K,j}^2\|d_j(\hat{z}_{K,j-1})\|^2  + \alpha_K M^2 \sum_{i=1}^n \left( \frac{\alpha_{K,i}}{\alpha_{K}} - 1 \right)^2, 
				\end{align*}
				which is \eqref{eq:smoothAdaTemp8}
\end{proof}

\begin{claim}
				$F$ has $L$ Lipschitz gradient.
				\label{claim:FSmooth}
\end{claim}
\begin{proof}
				For any $x,y$, we have
				\begin{align*}
								\|\nabla F(x) - \nabla F(y)\| &= \frac{1}{n} \left\|\sum_{i=1}^n d_i(x) - d_i(y) \right\| \leq \frac{1}{n} \sum_{i=1}^n \left\| d_i(x) - d_i(y) \right\| \leq \frac{1}{n} \sum_{i=1}^n L_i\left\| x - y \right\|\\
								&= L \| x - y \|
				\end{align*}
				where we used triangle inequality and $L_i$ Lipschicity of $d_i$.
\end{proof}

\begin{proof}[of Claim \ref{claim:mainIneqSmooth}]

				Using smoothness of $F$ in Claim \ref{claim:FSmooth}, we have from the descent Lemma \cite[Lemma 1.2.3]{nesterov2004introductory}, for all $x,y \in \RR^p$
				\begin{align}
								F(y) \leq F(x) + \left\langle\nabla F(x), y - x \right\rangle + \frac{L}{2} \|y-x \|^2.
								\label{eq:descentLemma}
				\end{align}
				Choosing $y = x_{K+1}$ and $x = x_K$ in \eqref{eq:descentLemma}, using Claim \ref{claim:smoothAdaTemp2} and Claim \ref{claim:iterateBound}, we obtain
				\begin{align*}
								F(x_{K+1}) \leq \,& F(x_K) + \left\langle \nabla F(x_K), x_{K+1} - x_K \right\rangle + \frac{L}{2} \|x_{K+1} - x_K\|^2\\
								\leq \,&F(x_K) - \frac{n \alpha_K}{2} \|\nabla F(x_K) \|^2 + \alpha_K L^2n^2\sum_{j=1}^n \alpha_{K,j}^2\|d_j(\hat{z}_{K,j-1})\|^2 + \alpha_K M^2 \sum_{i=1}^n \left( \frac{\alpha_{K,i}}{\alpha_{K}} - 1 \right)^2 \\
								&+ \left(\frac{L}{2} - \frac{1}{2n\alpha_K}\right)\|x_{K+1} - x_K\|^2 \\
								\leq \,&F(x_K) - \frac{n \alpha_K}{2} \|\nabla F(x_K) \|^2 + \left(\alpha_K L^2n^2+ \left(\frac{Ln}{2} - \frac{1}{2 \alpha_K}\right)_+\right)\sum_{j=1}^n \alpha_{K,j}^2\|d_j(\hat{z}_{K,j-1})\|^2 \\
								&+ \alpha_K M^2 \sum_{i=1}^n \left( \frac{\alpha_{K,i}}{\alpha_{K}} - 1 \right)^2,
				\end{align*}
				where the last inequality is obtained by Lemma \ref{lem:techLemma}.
				Since $\alpha_{K,i} \leq \alpha_K$ for all $K \in \NN$ and $i =1 \ldots, n$, we have $0 \leq \alpha_{K,i} / \alpha_K \leq 1$, and using $(t-1)^2 \leq 1 - t^2$ for all $t \in [0,1]$
				\begin{align*}
								\left( \frac{\alpha_{K,i}}{\alpha_{K}} - 1 \right)^2 &\leq 1 - \frac{\alpha_{K,i}^2}{\alpha_K^2}\leq 1 - \frac{\alpha_{K,i}^3}{\alpha_K^3}, 
				\end{align*}
				and the result follows.
\end{proof}

\section{Proofs for the nonsmooth setting}
\label{sec:proofsNonSmooth}

\begin{proof}[of Theorem \ref{th:PDITAPT}]
				Fix $T > 0$, we consider the sequence of functions, for each $k \in \NN$
				\begin{align*}
								\bw_k \colon [0,T] &\mapsto \RR^p \\
								t &\mapsto \bw(\tau_k + t)
				\end{align*}
				From Assumption \ref{ass:conservative} and Definition \ref{def:perturbedSolution}, it is clear that all functions in the sequence are $M$ Lipschitz. Since the sequence $(x_k)_{k \in \NN}$ is bounded, $(\bw_k)_{k\in \NN}$ is also uniformly bounded, hence by Arzel\`a-Ascoli theorem \cite[Chapter 10, Lemma 2]{royden1988real}, there is a a subsequence converging uniformly, let $\bz \colon [0,T] \mapsto \RR^p$ be any such uniform limit. By discarding terms, we actually have $\bw_k \to \bz$  as $k \to \infty$, uniformly on $[0,T]$. Note that we have for all $t \in [0,1]$, and all $\gamma>0$
				\begin{align}
								D^{\gamma}(\bw_k(t)) \subset D^{\gamma + \|\bw_k - \bz\|_\infty}(\bz(t)).
								\label{eq:PDIAPT0}
				\end{align}
				For all $k \in \NN$, we set $\bv_k \in L^2([0,T], \RR^p)$ such that $\bv_k = \bw_k'$ at points where $\bw_k$ is differentiable (almost everywhere since it is piecewise affine). We have for all $k \in \NN$ and all $s \in [0,T]$
				\begin{align}
								\bw_k(s) - \bw_k(0) = \int_{t=0}^{t=s} \bv_k(t)dt,	
								\label{eq:PDIAPT1}
				\end{align}
				and from Definition \ref{def:perturbedSolution}, we have for almost all $t \in [0,T]$,
				\begin{align}
								\bv_k(t) \in - D^{\gamma(\tau_k + t)}(\bw_k(t)).
								\label{eq:PDIAPT2}
				\end{align}
				Hence, the functions $\bv_k$ are uniformly bounded thanks to Assumption \ref{ass:conservative} and hence the sequence $(\bv_k)_{k\in\NN}$ is bounded in $L^2([0,T], \RR^p)$ and by Banach-Alaoglu theorem \cite[Section 15.1]{royden1988real}, it has a weak cluster point. Denote by $\bv$ a weak limit of $\left( \bv_k \right)_{k \in \NN}$ in $L^2([0,T], \RR^p)$. Discarding terms, we may assume that $\bv_k \to \bv$ weakly in $L^2([0,T], \RR^p)$ as $k \to \infty$ and hence, passing to the limit in \eqref{eq:PDIAPT1}, for all $s \in [0,T]$,
				\begin{align}
								\bz(s) - \bz(0) = \int_{t=0}^{t=s} \bv(t)dt.
								\label{eq:PDIAPT21}
				\end{align}
				By Mazur's Lemma (see for example \cite{ekeland1976convex}), there exists a sequence $(N_k)_{k \in \NN}$, with $N_k \geq k$ and a sequence $\tilde{\bv}_{k \in \NN}$ such that for each $k \in \NN$, $\tilde{\bv}_k \in \conv\left( \bv_k,\ldots, \bv_{N_k} \right)$ such that $\tilde{\bv}_k$ converges strongly in $L^2([0,T], \RR^p)$ hence pointwise almost everywhere in $[0,T]$. Using \eqref{eq:PDIAPT2} and the fact that countable intersection of full measure sets has full measure, we have for almost all $t \in [0,T]$
				\begin{align*}
								\bv(t) = \lim_{k \to \infty} \tilde{\bv}_k(t)&\in \lim_{k \to \infty} -\conv\left( \cup_{j=k}^{N_k} D^{\gamma(\tau_j + t)}(\bw_j(t)) \right) \\
								& \subset \lim_{k \to \infty} -\conv\left( \cup_{j=k}^{N_k} D^{\gamma(\tau_j + t) + \|\bw_j - \bz\|_\infty}(\bz(t)) \right)\\
								&= - \conv\left( \frac{1}{n} \sum_{i=1}^n D_i(\bz(t)) \right) = -D(\bz(t)).
				\end{align*}
				where we have used \eqref{eq:PDIAPT0}, the fact that $\lim_{\gamma \to 0} D^\gamma = \frac{1}{n} \sum_{i=1}^n D_i$ pointwise since each $D_i$ has closed graph and the definition of $D$. Using \eqref{eq:PDIAPT21}, this shows that for almost all $t \in [0,T]$, 
				\begin{align*}
								\dot{\bz}( t ) = \bv(t) \in - D(\bz(t)).
				\end{align*}
				Using \cite[Theorem 4.1]{benaim2005stochastic}, this shows that $\bw$ is an asymptotic pseudo trajectory.
\end{proof}

\section{Lemmas and additional proofs}
\label{sec:lemmas}

\begin{lemma}
				Let $a_1,\ldots, a_m$ be vectors in $\RR^p$, then
				\begin{align*}
								\left\| \sum_{i=1}^m a_i \right\|^2 \leq m \sum_{i=1}^m \|a_i\|^2
				\end{align*}
				\label{lem:techLemma}
\end{lemma}
\begin{proof}
				From the triangle inequality, we have
				\begin{align*}
								\left\| \sum_{i=1}^m a_i \right\|^2 \leq \left( \sum_{i=1}^m \|a_i\| \right)^2
				\end{align*}
				Hence it suffices to prove the claim for $p=1$. Consider the quadratic form on $\RR^m$
				\begin{align*}
								Q \colon x \mapsto m \sum_{i=1}^m x_i^2 - \left\| \sum_{i=1}^m x_i \right\|^2.
				\end{align*}
				We have 
				\begin{align*}
								Q(x) = m (\|x\|^2 - \left( x^T e \right)^2),
				\end{align*}
				where $e \in \RR^m$ has unit norm and with all entries equal to $1/\sqrt{m}$. The corresponding matrix is $m(I - e e^T)$ which is positive semidefinite. This proves the result.
\end{proof}

\begin{lemma}
				Let $(a_k)_{k \in \NN}$ be a sequence of positive numbers, and $b,c>0$. Then for all $m \in \NN$
				\begin{align*}
								\sum_{i=0}^m \frac{a_i}{b + c \sum_{j=0}^i a_j}  \leq \frac{1}{c}\log\left( 1 + c \frac{\sum_{i=0}^m a_i}{b} \right)  
				\end{align*}
				\label{lem:sumStepSizeLog}
\end{lemma}
\begin{proof}
				We have
				\begin{align*}
								\sum_{i=0}^m \frac{a_i}{b + c \sum_{j=0}^i a_j} &= \frac{1}{ c} \sum_{i=0}^m \frac{a_i}{\frac{b}{c} + \sum_{j=0}^i a_j} \\
								&\leq \frac{1}{ c} \log\left( 1 + c \frac{\sum_{i=0}^m a_i}{b} \right)
				\end{align*}
				where the last inequality follows from Lemma 6.2 in \cite{defossez2020convergence}.
\end{proof}

\end{document}